\documentclass[letterpaper]{article}
\usepackage{proceed2e}
\usepackage[margin=1in]{geometry}
\usepackage{times}
\usepackage{amssymb}
\usepackage{graphicx}
\usepackage[fleqn]{amsmath}
\usepackage{amsthm}
\usepackage{algpseudocode}
\usepackage{algorithm}
\usepackage{amsthm}
\usepackage{colortbl}
\usepackage[utf8]{inputenc}
\usepackage{xr}
\usepackage{url}
\usepackage[square]{natbib}

\DeclareMathOperator*{\argmax}{arg\,max}

\usepackage{mathtools}  
\usepackage{tabulary}
\usepackage{booktabs}

\theoremstyle{plain}
\newtheorem{theorem}{Theorem}

\newtheorem{lemma}{Lemma}
\newtheorem*{lemma*}{Lemma}
\newtheorem*{theorem*}{Theorem}
\newtheorem{assumption}{Assumption}
\theoremstyle{definition}
\theoremstyle{remark}

\def\mathhyphen{{\hbox{-}}}

\title{Real-Time Resource Allocation for Tracking Systems}
\author{ {\bf Yash Satsangi  } \\
University of Amsterdam \\
\And
{\bf Shimon Whiteson  } \\
University of Oxford \\
\And
{\bf Frans A. Oliehoek} \\
University of Liverpool \\ 
University of Amsterdam
\And
{\bf Henri Bouma }  \\
TNO, The Netherlands \\
}

\begin{document}
\maketitle

\begin{abstract}
Automated tracking is key to many computer vision applications. However, many tracking systems struggle to perform in real-time due to the high computational cost of detecting people, especially in ultra high resolution images. We propose a new algorithm called \emph{PartiMax} that greatly reduces this cost by applying the person detector only to the relevant parts of the image. PartiMax exploits information in the particle filter to select $k$ of the $n$ candidate \emph{pixel boxes} in the image. We prove that PartiMax is guaranteed to make a near-optimal selection with error bounds that are independent of the problem size. Furthermore, empirical results on a real-life dataset show that our system runs in real-time by processing only 10\% of the pixel boxes in the image while still retaining 80\% of the original tracking performance achieved when processing all pixel boxes.

\end{abstract}

\section{INTRODUCTION}

Automated tracking is a key component of countless computer vision applications such as maintaining surveillance, studying traffic flows, and counting the number of people in a scene \citep{surveySmeulders}. Consequently, in recent years many tracking systems have been proposed that make it possible to track people in a variety of challenging settings \citep{illumTrack,Benfold2011,surveySmeulders}. However, these approaches still cannot perform real-time tracking on ultra high resolution videos (e.g., $5000 \times 4000$ pixels). 

In particular, the \emph{detection} stage, i.e., identifying an object in a scene, is the main computational bottleneck for systems that work on the \emph{tracking-by-detection} principle \citep{Benfold2011}. For example, Figure~\ref{tnoScene} shows a wide-view scene recorded by a camera mounted on top of a building~\citep{Schutte2016}. Successful tracking depends on detecting the person in the image by applying a trained detector to many \emph{pixel boxes}. Since the scene records a wide landscape, the pixel boxes must be relatively small (e.g., $180 \times 180$), yielding approximately 7000 pixel boxes per image. Consequently, performing a \emph{brute force detection} (BD) that applies the person detector to all 7000 pixel boxes is extremely computationally intensive and prohibitive to do in real time.

\begin{figure}
\begin{center}
\includegraphics[scale=0.35]{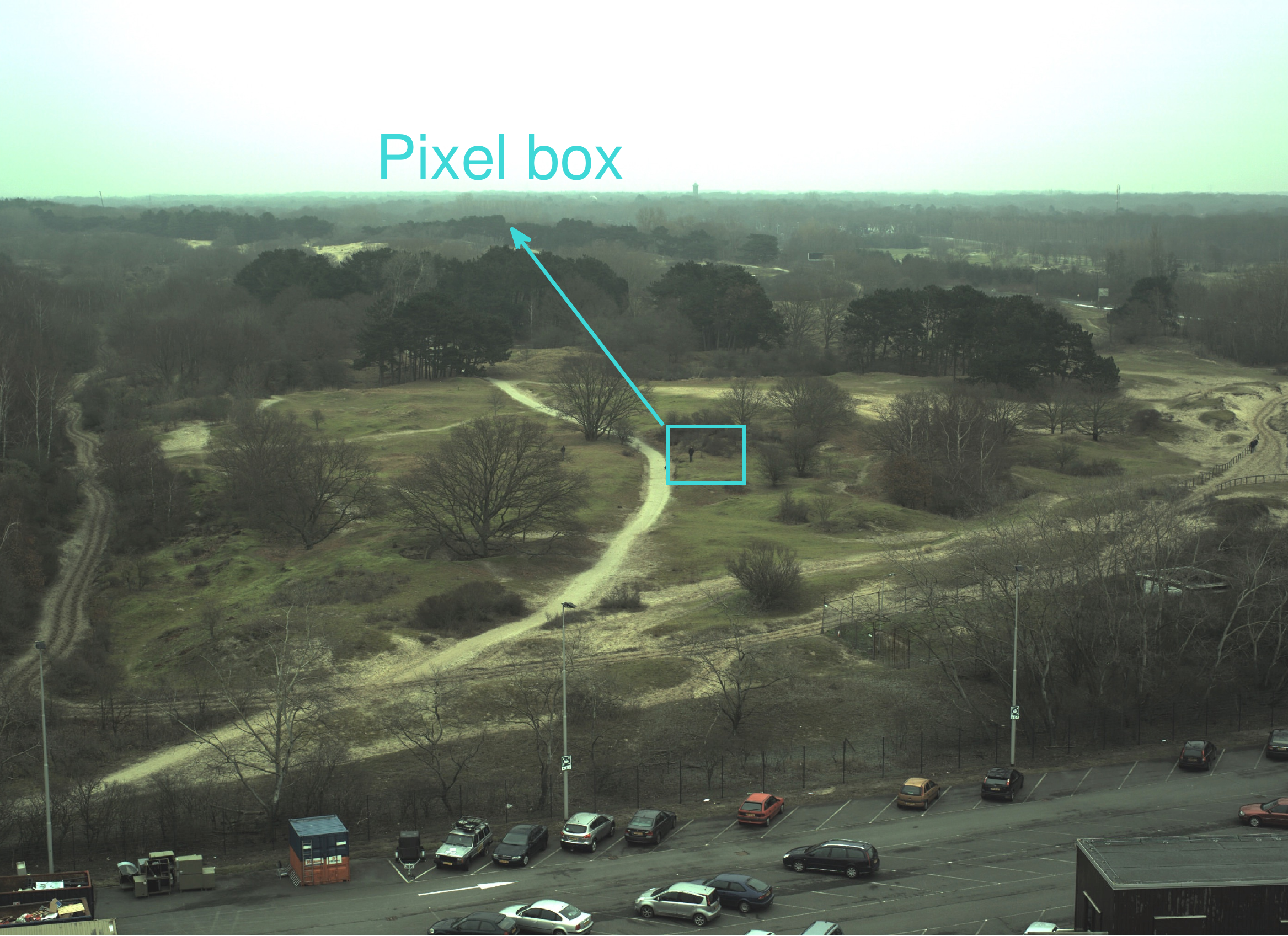}
\caption{A wide-view scene recorded by a rooftop camera; the cyan rectangle shows an example pixel box.} 
\label{tnoScene}
\end{center}
\end{figure}

\begin{figure}
\begin{center}
\vspace{-4mm}
\includegraphics[scale=0.35]{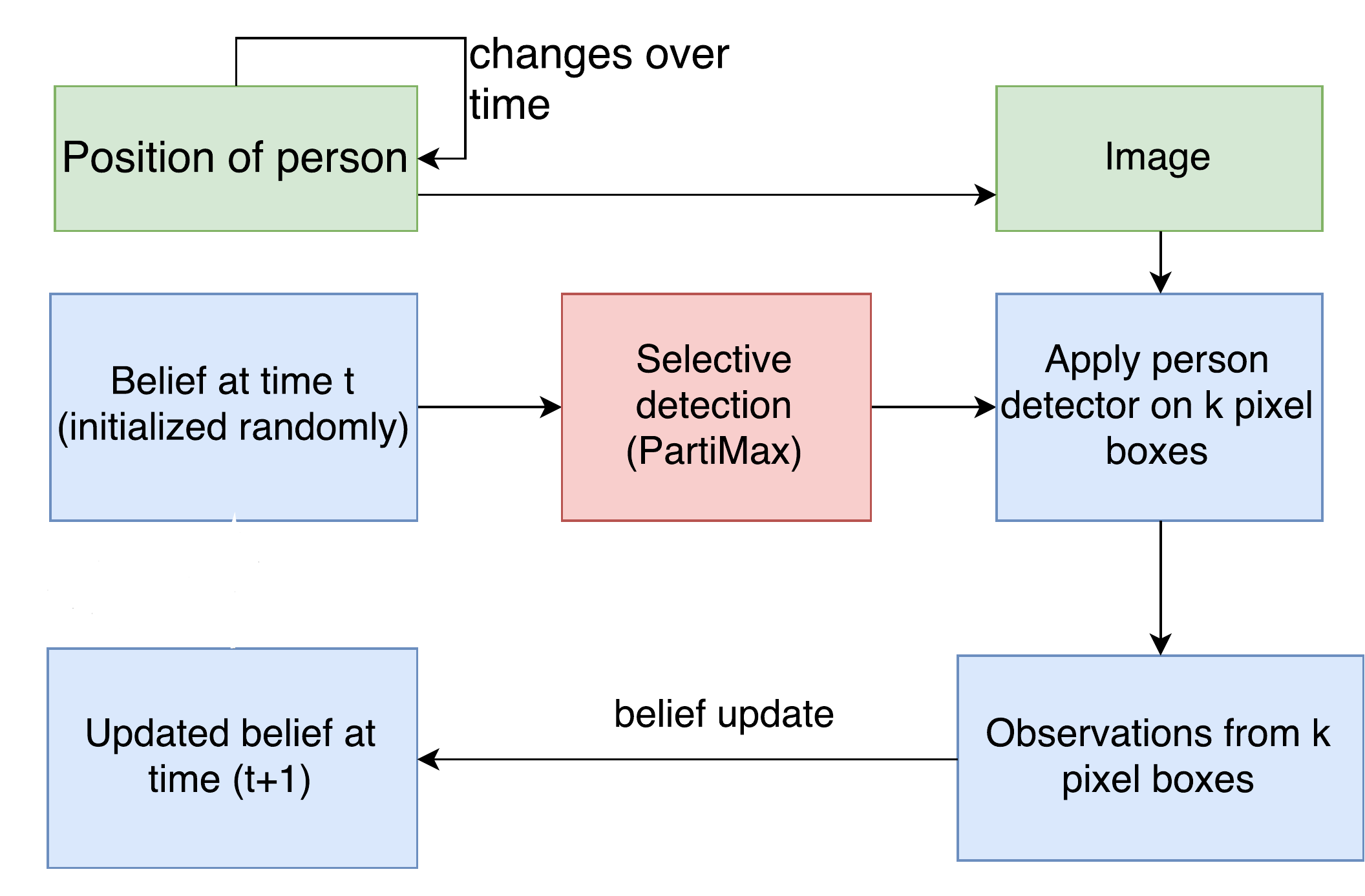}
\caption{Proposed tracking system with PartiMax, our proposed selective detection method, highlighted in red.} \label{fig:SD}
\end{center}
\end{figure}

In this paper, we propose a new tracking system that greatly reduces the cost of detection and thus enables real-time tracking on systems with ultra high resolution images or many cameras.  The main idea is to perform \emph{selective detection} (SD), i.e., apply the person detector not on all $n$ pixel boxes, but only a carefully selected subset of $k$ pixel boxes, while retaining performance guarantees, as shown in Figure \ref{fig:SD}. To do so, we build on existing techniques for \emph{sensor selection}, which select the $k$ out of $n$ sensors with the highest \emph{utility} in a multi-sensor network.  Sensor selection is challenging because there are $n \choose k$ ways to perform the selection, and computing the best one would use up the same scarce computational resources we aim to intelligently allocate.  Fortunately, when the utility function possesses certain characteristics, including \emph{submodularity}, a near-optimal selection can be found using \emph{greedy maximization}, which evaluates the utility function only $\mathcal{O}(n k)$, instead of $n \choose k$, times. In addition, \emph{stochastic greedy maximization} \citep{lazier} further reduces the number of evaluations of the utility function by evaluating the utility function only for a random subset of pixel boxes of size $r$, where $r << n$.

However, for selective detection in real-time, even stochastic greedy maximization is too expensive because computing typical utility functions such as \emph{information gain} or \emph{expected coverage} requires marginalizing out the observation that each candidate sensor would generate. In fact, in real-life settings with high dimensional state and/or observation spaces, evaluating information gain or expected coverage even once can be prohibitively expensive.

We start by proposing a utility function for selective detection called \emph{particle coverage} that approximates the probability of detecting a person in a given set of pixel boxes. We show that particle coverage approximates expected coverage under certain conditions, but is much faster to compute. Then, we propose \emph{PartiMax}. Unlike (stochastic) greedy maximization, which treats utility evaluation as a black-box, PartiMax maintains and updates the particle coverage of each pixel box in every iteration of greedy maximization, leading to large computational savings, as the particle coverage of each pixel box is not evaluated from scratch in each iteration. 
Furthermore, instead of selecting a subset of pixel boxes \emph{randomly} in every iteration like stochastic greedy maximization, PartiMax samples pixel boxes with high particle coverage leading to superior tracking performance.

Since sampling pixel boxes with high particle coverage without computing the particle coverage is not trivial, we propose a sampling algorithm that is guaranteed to sample a pixel box with probability directly proportional to its particle coverage. It does so by employing \emph{tile coding}, a popular representation in reinforcement learning that discretizes continuous spaces.

We show that, given access to a sampling algorithm like the one we propose, PartiMax is guaranteed to return a solution with tight error bounds that are independent of the problem size, i.e., independent of both $n$ and $k$. Although PartiMax is designed for the particle coverage function, our bound applies generally to maximization over a set function.

Finally, we use PartiMax for selective detection to build a real-time tracking system, which we apply to a real-life dataset. Our results show that our tracking system retains 80\% of its performance despite processing only 10\% of each image and running in real time.

\section{RELATED WORK}
Most detection systems, e.g., \citep{Felzenszwalb2010,Dollar2014}, including those based on convolutional neural networks \citep{Tian2015,Redmon2016}, process the whole image and are thus not computationally efficient enough for our setting, due to the high resolution and depth of scene of the images.

Some work does identify the relevant region of interest in an image
\citep{detectingROI}, e.g., by generating \emph{proposals}
(see \citep{Hosang2015} and references therein) or saliency points \citep{saliency}. These methods, however, are based on the
properties (or low-level features) of the entire image (since they do
not consider the belief state) and often generate thousands of
proposals/saliency points per image. In fact, selective detection can be coupled with these methods to selectively generate saliency points.

Recently developed models of visual attention \citep{recurrent,learningWTL} come closest to our work in spirit. However, they use model-free deep reinforcement learning methods to identify relevant region to apply a trained detector on \citep{recurrent}, while we learn the model of the world from the data and use it to plan online to find the relevant regions to which to apply a trained detector.

Our work builds off the vast existing sensor selection literature.  Most work uses utility functions involving information gain \citep{InformationSS,inforTheoSS,pacGreedy} and expected coverage \citep{spaan}, which are too expensive for real-time systems. Other approaches do not consider partial observability \citep{natarajan} or do not scale to large state and observation spaces \citep{natarajan,satsangi15}. Methods based on dynamic programming \citep{dynamicPro} or linear programming \citep{sensorManage} or focusing on occlusions \citep{costApproach} are also limited to smaller state and observation spaces.

For submodular function maximization, the most related methods are those of \citep{lazier} and \citep{thresholdgreedy}. We significantly improve upon these methods for sensor selection by introducing a novel method with lower computational cost and thereby making them applicable to real-time tracking.

\section{BACKGROUND}

\subsection{BASIC SETUP}

Let $\mathcal{X} = \{1,2 \dots n\}$ denote the set of all pixel boxes and $i$ denote a single pixel box in $\mathcal{X}$. $\mathcal{A}^{+}$ denotes the set of all possible subsets of $\mathcal{X}$ of size less than or equal to $k$, $\mathcal{A}^{+} = \{\mathcal{A} \subseteq \mathcal{X}: |\mathcal{A}| \leq k\}$. For the image shown in Figure \ref{tnoScene} the size of one pixel box was chosen to be 180 x 180 pixels. The true location of the person is a hidden variable denoted by $s$ and $S$ is the set of all possible values $s$ can take.\footnote{For simplicity, we sometimes assume there is only one person in the scene and the hidden variable is a vector in the Euclidean space. However, our methods and theoretical results extend easily to multiple people, as shown in Section \ref{sec:exp}. }  The observation vector $\mathbf{z} = \langle z_1, z_2 \dots z_{n}\rangle$ denotes the result of applying the detector to each pixel box, i.e., each $z_i$ denotes an estimate of whether a person appears in the pixel box $i$. If a pixel box $i$ is not selected for detection, then $z_i = \emptyset$. $\Omega$ is the set of all possible values $\mathbf{z}$ can take. The \emph{belief} $b(s)$ is a probability distribution over $s$. Given $\mathcal{A}$ and $\mathbf{z}$, $b(s)$ can be updated using Bayes rule. 

When there are many possible states, it is not possible to maintain $b(s)$ exactly. Thus, we use particle filters, described below, to maintain and update belief $b(s)$. Below we also describe greedy maximization, which is essential to our setup, as it selects $k$ pixel boxes out of $n$ in $\mathcal{O}(n \times k)$ time instead of $\mathcal{O}$$n \choose k$. 


\subsection{PARTICLE FILTERS}

When there are many possible states, it is infeasible to update $b(s)$ exactly. Instead, we can use \emph{particle filters} \citep{particlef}, sequential Monte Carlo algorithms for approximate inference in partially observable scenarios that are commonly used to track people in complex situations. The true belief $b(s)$ is approximated with a \emph{particle belief} $\mathcal{B}$, a collection of $m$ samples from $b(s)$, called particles: $\mathcal{B} = \{s_1, s_2 \dots s_m \}$. Although weighted particle filters are often used for tracking, we use an unweighted particle filter since it can be  efficiently implemented with a black-box simulator without the need to explicitly model the accuracy of the person detector or the motion dynamics.

Given the particle belief $\mathcal{B}$, a subset of sensors $\mathcal{A}$ and observation $\mathbf{z}$, particle beliefs can be updated using a \emph{Monte Carlo belief update} \citep{pomcp}. For each particle $s_l \in \mathcal{B}$, the next state $s_l'$ is sampled from $\Pr(s'|s)$ (under the Markov assumption) to form $\mathcal{B}': \{s_l' : s_l' \sim \Pr(s_l'|s_l) \wedge s_l \in \mathcal{B} \}$. With enough samples, $\mathcal{B}'$ approximates the probability distribution: $b'(s') = \sum_{s \in S} \Pr(s'|s)b(s)$.

For each $s_l' \in \mathcal{B}'$, the corresponding $\mathbf{z}_l$ is drawn from $\Pr(\mathbf{z}|s_l',\mathcal{A})$. If $\mathbf{z}_l=\mathbf{z}$, then $s'_l$ is added to the updated belief $\mathcal{B}^{\mathcal{A}}_{\mathbf{z}}$. Otherwise, the particle is discarded.  To avoid particle degeneracy, a common problem with particle filters, we combine the belief update with new particles introduced by adding random particles sampled from $S$ to the existing particle set. $\mathcal{B}^{\mathcal{A}}_{\mathbf{z}}$ approximates the probability distribution $b^{\mathcal{A}}_{\mathbf{z}}(s') = \frac{\Pr(\mathbf{z}| \mathcal{A}, s') b'(s')}{\Pr(\mathbf{z}|b, \mathcal{A})}$.

\subsection{GREEDY MAXIMIZATION}
Given a set function $F(\mathcal{A})$, where $\mathcal{A} \in \mathcal{A}^+$, \emph{greedy maximization} computes $\mathcal{A}^{G}$, which approximately maximizes $F$ by building a subset of $k$ pixel boxes iteratively. In particular, in each of its $k$ iterations, greedy maximization adds to a partial solution the pixel box that maximizes the \emph{marginal gain}:
\begin{equation}
\Delta_{F}(i|\mathcal{A}) = F(\mathcal{A} \cup i) - F(\mathcal{A}),
\end{equation}
of adding $i$ to $\mathcal{A}$, i.e., it adds $\argmax_{i \in \mathcal{X} \setminus \mathcal{A}^{G}}\Delta_{F}(i|\mathcal{A}^{G})$ to $\mathcal{A}^{G}$ as shown in Algorithm \ref{GM}. 

\begin{algorithm}
\caption{$\mathtt{greedyMax}(F,\mathcal{X},k)$}\label{GM}
\begin{algorithmic}[1]
\State $\mathcal{A}^{G} \gets \emptyset$
\For {$l = 1 \ to \ k$}
\State $\mathcal{A}^{G} \gets \mathcal{A}^{G} \cup \argmax_{i \in \mathcal{X} \setminus \mathcal{A}^{G}}\Delta_{F}(i|\mathcal{A}^{G})$
\EndFor
\State return $\mathcal{A}^{G}$
\end{algorithmic}
\end{algorithm}

\citet{nemhauser} showed that greedy maximization is guaranteed to have bounded error under certain conditions: 
\begin{theorem} \citep{nemhauser} \label{th:nemhauser}
If $F$ is non-negative, monotone and submodular%
, then
$F(\mathcal{A}^{G}) \geq (1 - e^{-1}) F(\mathcal{A}^{*})$, where $\mathcal{A}^* = \arg\max_{\mathcal{A} \in \mathcal{A}^+}F(A)$.
\end{theorem}

Submodularity is a property of set functions that formalizes the notion of diminishing returns: $F : 2^{\mathcal{X}} \to \mathbb{R}$ is submodular if for every $\mathcal{A}_{M} \subseteq \mathcal{A}_{N} \subseteq \mathcal{X}$ and $i \in \mathcal{X} \setminus \mathcal{A}_N$,
\begin{equation}
\Delta_{F}(i|\mathcal{A}_{M}) \geq \Delta_{F}(i|\mathcal{A}_{N}).
\end{equation}
 
Thus, the marginal gain of adding an element to a smaller set $\mathcal{A}_{M}$ is always greater than or equal to the marginal gain of adding the same element to a bigger subset $\mathcal{A}_{N}$ such that $\mathcal{A}_{M} \subseteq \mathcal{A}_{N} \subseteq \mathcal{X}$. If this is true for all possible values of $\mathcal{A}_{N}, \mathcal{A}_{M}$, and $i$, then $F$ is submodular.

\subsection{STOCHASTIC GREEDY MAXIMIZATION}

\emph{Stochastic greedy maximization}, shown in Algorithm \ref{stoGM}, further reduces costs by randomly sampling a subset $\mathcal{R}$ of size $r$ from $\mathcal{X}$ in each iteration of greedy maximization and then selecting the element from $\mathcal{R}$ that maximizes the marginal gain. It computes a subset 
$\mathcal{A}^{S}$ by adding in each iteration $\argmax_{i \in \mathcal{R}}\Delta_{F}(i|\mathcal{A}^{S})$, where $\mathcal{R}$ is a subset of $\mathcal{X} \setminus \mathcal{A}^{S}$ of size $r$. \citet{lazier} showed that stochastic greedy maximization is also guaranteed to have bounded error.

\begin{algorithm} 
\caption{$\mathtt{stochastic}\mathhyphen\mathtt{greedy}\mathhyphen\mathtt{max}(F,\mathcal{X},k,r)$}\label{stoGM}
\begin{algorithmic}[1]
\State $\mathcal{A}^{S} \gets \emptyset$
\For {$m = 1 \ to \ k$}
\State $\mathcal{R} \gets$ random sample of size $r$ from $\mathcal{X} \setminus \mathcal{A}^{S}$.
\State $\mathcal{A}^{S} \gets \mathcal{A}^{S} \cup \argmax_{i \in \mathcal{R}}\Delta_{F}(i|\mathcal{A}^{S})$
\EndFor
\State return $\mathcal{A}^{S}$
\end{algorithmic}
\end{algorithm}

\begin{theorem} \citep{lazier} \label{thm:lazier}
If $F$ is non-negative, monotone and submodular%
, then
$\mathbb{E}[F(\mathcal{A}^{S})] \geq (1 - e^{-1} - \epsilon) F(\mathcal{A}^{*})$,
where $r = \frac{n}{k}\log(\frac{1}{\epsilon})$.
\end{theorem}

\subsection{UTILITY FUNCTIONS}
For tracking tasks, $F$ is often defined as \emph{information gain} \citep{cover,igkrause,InformationSS}:
\begin{equation}
	IG_{b}(\mathcal{A}) = H_{b}(s) - H_{b}^{\mathcal{A}}(s|\mathbf{z}),
\end{equation}	
where $H_{b}(s)$ is the \emph{entropy} of $s$ and $H_{b}^{\mathcal{A}}(s|\mathbf{z})$ is the \emph{conditional entropy} of $s$ given $\mathbf{z}$ \citep{cover}.

It can also be defined as \emph{expected coverage} \citep{spaan}. Let $\mathcal{I}_{\mathcal{B}'}^j$ be the set of particles in $\mathcal{B}'$ that are \emph{covered} by pixel box $j$, $\mathcal{I}^{j}_{\mathcal{B}'} = \{s' \in \mathcal{B}': \mbox{$j$ covers $s'$} \}$. A pixel box $j$ covers $s'$ if a person in state $s'$ is visible in pixel box $j$. The expected coverage is defined as:
\begin{equation} \label{eq:expCov}
F_{\mathcal{B}'}(\mathcal{A}) = \sum_{\mathbf{z}}\Pr(\mathbf{z}|\mathcal{B}',\mathcal{A})f_{\mathcal{B}^{\mathcal{A}}_{\mathbf{z}}}(\mathcal{A}),
\end{equation}
where $f_{\mathcal{B}}(\mathcal{A}) = |\cup_{j \in \mathcal{A}} \mathcal{I}^{j}_{\mathcal{B}}|$. Expected coverage belongs to a general class of coverage functions that have been widely considered \citep{spaan}. In tracking, expected coverage is suitable because of the presence of partial observability, necessitating the expectation across $\mathbf{z}$. Expected coverage is appropriate for sensor selection or selective detection because it rewards selecting pixel boxes that have the highest probability of detecting a target. The underlying assumption is that the observations generated by the person detector are informative enough to detect a person correctly when present inside the pixel box, and are not informative enough if a person is absent from the pixel box. This is barely a restrictive assumption, as most useful person detectors satisfy it.

\section{PARTICLE COVERAGE UTILITY FUNCTION}

The utility functions described above are too expensive to compute in many practical settings, as they require marginalizing out observations, which is infeasible for real-time systems. In this section, we propose the \emph{particle coverage function} (PCF) for selective detection, which does not require computing $\mathcal{B}^{\mathcal{A}}_{\mathbf{z}}$ and approximates expected coverage.  PCF is defined as follows:
\begin{equation}
PCF_{\mathcal{B}'}(\mathcal{A}) = f_{\mathcal{B'}}(\mathcal{A}) = |\cup_{j \in \mathcal{A}}\mathcal{I}^{j}_{\mathcal{B}'}|.
\end{equation}

\begin{figure}
\begin{center}
\includegraphics[scale=0.6]{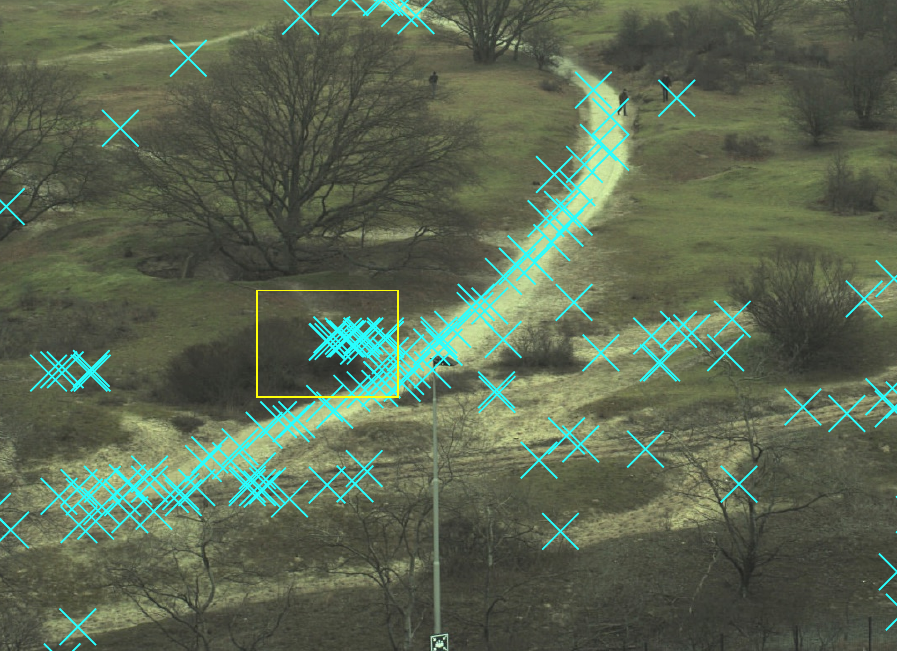}
\caption{Particle belief: the yellow rectangle shows a pixel box and the particles it covers.}
\label{fig:ut}
\end{center}
\end{figure}

$PCF_{\mathcal{B}'}(\mathcal{A})$ is simply the number of particles in $\mathcal{B'}$ that are covered by $\mathcal{A}$. In Figure~\ref{fig:ut}, the particle coverage is the number of cyan particles that fall in the yellow pixel box. As opposed to expected coverage $F_{\mathcal{B}'}$, PCF does not involve an expectation over $\mathbf{z}$ nor does it require computing the resulting beliefs $\mathcal{B}_{\mathbf{z}}^{\mathcal{A}}$. 
PCF equals expected coverage under certain conditions, including the following.

\begin{assumption} \label{as:as2}
For every $s' \in S$, $\mathcal{A} \subseteq \mathcal{X}$, there exist $\mathbf{z}_{s',\mathcal{A}}$ and $\mathbf{\bar{z}}_{s',\mathcal{A}}$ in $\Omega$ such that if $s'$ is covered by $\mathcal{A}$, $\Pr(\mathbf{z}_{s',\mathcal{A}}|s',\mathcal{A}) = 1$ and if $s'$ is not covered by $\mathcal{A}$, then $\Pr(\mathbf{\bar{z}}_{s',\mathcal{A}}|s',\mathcal{A}) = 1$.
\end{assumption}

This assumption implies that any partial observability is due to perceptual aliasing, not noise in the sensors.  Given Assumption \ref{as:as2}, it is straightforward to show that expected coverage is equal to the particle coverage. 
\begin{theorem} \label{th:pcufEqCov}
If Assumption \ref{as:as2} holds for a given $\mathcal{A}$, then $F_{\mathcal{B}'}(\mathcal{A}) = PCF_{\mathcal{B}'}(\mathcal{A})$. 
\end{theorem}
\begin{proof}
Expected coverage can be expressed as $F_{B'}(\mathcal{A}) = \sum_{\mathbf{z} \in \Omega}\Pr(\mathbf{z}|\mathcal{B'},\mathcal{A}) f_{\mathcal{B}^{\mathcal{A}}_{\mathbf{z}}}(\mathcal{A})$. In case a negative detection is observed, that is the person is not in the space covered by $\mathcal{A}$ the resulting belief will not have any particle within the space covered by $\mathcal{A}$ due to Assumption \ref{as:as2} and thus resulting coverage is zero. If a positive detection is observed, that is the person is inside the space covered by $\mathcal{A}$ then all the particles in resulting belief will fall within the space that is covered by $\mathcal{A}$ resulting:
$f_{\mathcal{B}^{\mathcal{A}}_{\mathbf{z}}}(\mathcal{A}) = m. $
This implies, 
$F_{B'}(\mathcal{A}) = \sum_{\mathbf{z} \in \Omega}\Pr(\mathbf{z}|\mathcal{B'},\mathcal{A}) m$, 
where $\mathbf{z}$ is a positive observation that can be obtained only if a state is covered by $\mathcal{A}$. The probability of getting a positive detection according to $\mathcal{B'}$ is the sum of particles covered by $\mathcal{A}$ in $\mathcal{B'}$ divided by $m$. Thus, $F_{\mathcal{B'}}(\mathcal{A}) = \frac{PCF_{\mathcal{B'}}{(\mathcal{A})}}{m} \times m = PCF_{\mathcal{B'}}(\mathcal{A})$.
\end{proof}

In cases where Assumption \ref{as:as2} does not hold, particle coverage can be considered an approximation to expected coverage.  Its key advantage is that computing $f_{\mathcal{B'}}$ does not require hypothetical belief updates, as one can iterate over the particle belief and simply count the number of particles that are covered by $\mathcal{A}$, making it practical for real-time applications. Moreover, it is a member of a class of coverage functions that are known to be submodular \citep{submodSurvey,textSumm} so we can employ greedy maximization to approximately maximize $f_{\mathcal{B'}}$. Our experiments show that $f_{\mathcal{B'}}$ is a good choice of utility function for selective detection in real time, leading to excellent tracking performance at a fraction of the computational cost.

Note that we formulate Assumption \ref{as:as2} merely for analysis purposes: to describe a set of cases in which particle coverage and expected coverage are identical.  Assumption \ref{as:as2} is not a restrictive condition for applying PartiMax, described below.  On the contrary, in the Experiments section we present excellent results for PartiMax on a real-life dataset for which Assumption \ref{as:as2} does not hold.

Furthermore, while we define particle coverage for the case of an unweighted particle filter, the concept is more general. In essence, the particle coverage of a pixel box is the cumulative probability mass concentrated on the states that are covered by the pixel box. Thus, any method that approximates a belief can be used to compute particle coverage by simply computing the probability mass concentrated on a set of states. For example, for a weighted particle filter, the particle coverage of a pixel box is just the sum of the weights of the particles covered by the pixel box. 

\vspace{-2mm}
\section{PARTIMAX}

In this section, we propose \emph{PartiMax}, which combines the complementary benefits of PCF and stochastic greedy maximization for selective detection. Moreover, rather than merely naively applying them together, we exploit the unique structure of PCF to develop a better approach for sampling pixel boxes that is guaranteed to sample pixel boxes with high coverage, thus offering a further increase in performance. PartiMax is based on the key insight that sampling pixel boxes with a probability that is directly proportional to their particle coverage leads to strong theoretical guarantees on the expected utility. Thus, we prove error bounds for PartiMax that are independent of the number of available pixel boxes $n$, the number of particles in the particle filter $m$, or the number of pixel boxes to be selected $k$.     

Greedy maximization and stochastic greedy maximization assume oracle access to the utility function and thus compute the marginal gain for every pixel box in every iteration. Generally, computing particle coverage function given a pixel box requires iterating over the particles to count how many fall in the space covered by the pixel box. Unlike greedy maximization, PartiMax does not explicitly compute particle coverage for each pixel box on the fly but instead maintains the particle coverage of each pixel box by updating it in every iteration. Using an approach inspired by \emph{tile coding} \citep{suttonBook}, a popular reinforcement learning technique for coding continuous state spaces, PartiMax is able to compute and maintain the particle coverage of every pixel box without having to visit $n$ pixel boxes or $m$ particles in every iteration. 

A tile coding consists of many \emph{tilings}. Each tiling is a set of \emph{tiles}, which in our setting are pixel boxes. The pixel boxes in a tiling partition the state space $S$, i.e., they are disjoint and completely cover $S$. For example, Figure~\ref{fig:tilings} shows two tilings in blue and yellow. Typically, different tilings have the same size pixel boxes but start at a fixed offset from each other, as in the figure. Since the pixel boxes in a given tiling form a partition, there is exactly one pixel box in each tiling that covers a given state $s'$. If we represent each tiling as an array, locating the pixel box that covers a given state $s'$ requires only simple arithmetic involving the size of the pixel boxes and the offset between the tilings. Figure~\ref{fig:tilings} highlights the two pixel boxes, one in each tiling, that cover a given state (red cross). Thus, by representing the entire space of pixel boxes as multiple tilings, the set of pixel boxes $\mathcal{T}_{s'}$ that covers a given state $s'$ can be identified in constant time. 

In reinforcement learning, tile codings are used to discretize continuous state spaces in order to approximate a value function. Here, we use it differently, just as a scheme for dividing an image into overlapping pixel boxes.  The benefit of this approach is that it enables PartiMax to maintain $\Delta_f$ efficiently, by providing constant-time access to the set $\mathcal{T}_{s'}$ of all pixel boxes that cover a given state $s'$, i.e., $\mathcal{T}_{s'} = \{i \in \mathcal{X}: \mbox{$i$ covers $s'$} \}$.

\begin{figure} 
\center \includegraphics[scale=0.26]{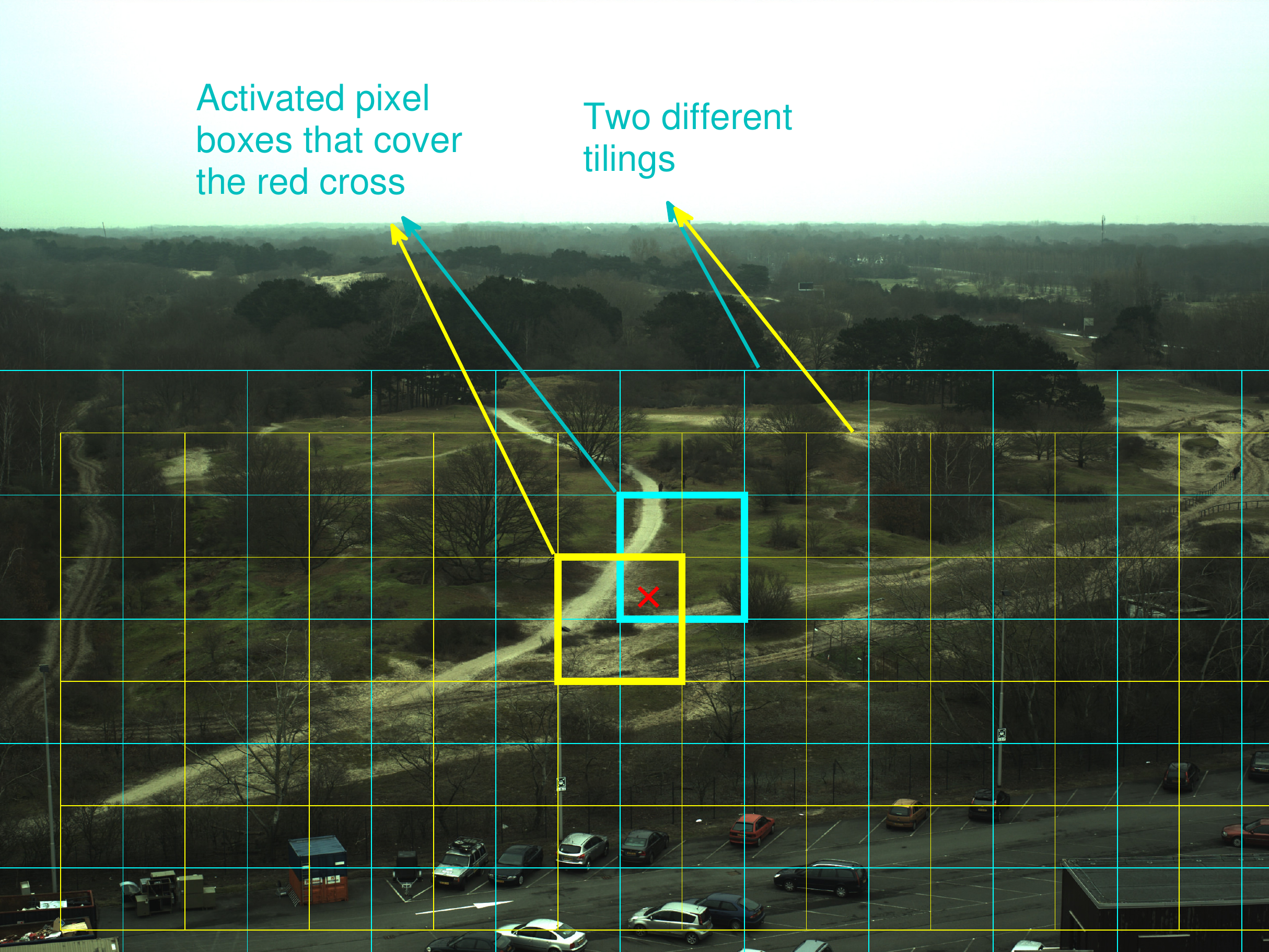}
\vspace{-2mm}
\caption{An example tile coding with two tilings. The highlighted tiles show the two pixel boxes that cover the red cross.}
\label{fig:tilings}
\end{figure}

\begin{algorithm}
\caption{$\mathtt{PartiMax}(\mathcal{B'},\mathcal{X},k)$}\label{PartiMaxMain}
\begin{algorithmic}[1]
\State $\langle \Phi, \Delta_{f} \rangle \gets \mathtt{initialize}(\mathcal{B'},\mathcal{X})$
\State $\mathcal{A}^{S} \gets \emptyset$.
\For {$l = 1 \ to \ k$}
\State $\mathcal{R} \gets \mathtt{sampleP}(r,\mathcal{B}',\mathcal{X},\mathcal{A}^{S})$
\State $i' \gets \argmax_{i \in \mathcal{R}}\Delta_{f}(i|\mathcal{A}^{S})$. \label{line:istar}
\State $\mathcal{A}^{S} \gets \mathcal{A}^{S} \cup i'$
\State $\langle \Delta_{f}, \Phi \rangle \gets \mathtt{update}(\Delta_{f},\Phi,i',\mathcal{A}^{S},\mathcal{X})$
\EndFor
\State return $\mathcal{A}^{S}$
\end{algorithmic}
\end{algorithm}

\begin{algorithm}
\caption{$\mathtt{initialize}(\mathcal{B'},\mathcal{X})$}\label{initialize}
\label{alg-inc}
\begin{algorithmic}[1]
\State $\Phi(i|\emptyset) \gets \emptyset \forall \ i \in \mathcal{X}$
\State $\Delta_f(i|\emptyset) \gets 0 \ \forall \ i \in \mathcal{X}$
\For {$s' \in \mathcal{B'}}$
\State $\mathcal{T}_{s'} \gets \mathtt{covers}(s')$
\State $\Phi(i|\emptyset) \gets \Phi(i|\emptyset) \cup \{s'\} \ \forall i \in \mathcal{T}_{s'}$ 
\State $\Delta_{f}(i|\emptyset) = \Delta_{f}(i|\emptyset) + 1  \ \forall i \in \mathcal{T}_{s'}$  
\EndFor
\State return $\langle \Phi, \Delta_{f} \rangle$
\end{algorithmic}
\end{algorithm}

Algorithm \ref{PartiMaxMain} shows pseudocode for PartiMax.  It starts by calling $\mathtt{initialize}$ (Algorithm \ref{alg-inc}), which returns two data structures, $\Phi$ and $\Delta_{f}$. $\Phi(i|\emptyset)$ stores for each $i$ the set of particles in $\mathcal{B}'$ that $i$ covers; and $\Delta_{f}(i|\emptyset)$ is the number of particles that are covered by $i$. For each particle $s' \in \mathcal{B'}$, $\mathtt{initialize}$ calls $\mathtt{covers}$, which uses the tile coding to find the set of pixel boxes $\mathcal{T}_{s'}$ that cover that particle. For every activated pixel box, $i \in \mathcal{T}_{s'}$, $\Delta_{f}(i|\emptyset)$ is incremented and $s'$ is added to the set of particles $\Phi(i|\emptyset)$. 

Once $\Phi$ and $\Delta_{f}$ are returned by $\mathtt{initialize}$, PartiMax proceeds like stochastic greedy maximization, adding in each iteration the pixel box $i'$ that maximizes the marginal gain from $\mathcal{R}$. Since going over all pixel boxes is too expensive, PartiMax calls Algorithm \ref{alg:sampleP2} to obtain $\mathcal{R}$, a subset of $\mathcal{X}$ of size $r$ ($r << n$). However, unlike stochastic greedy maximization, $\mathcal{R}$ is not sampled uniformly randomly but instead Algorithm \ref{alg:sampleP2} samples from a distribution such that the probability that $i$ is included in $\mathcal{R}$ is directly proportional to the particle coverage of $i$. 

\begin{algorithm}
\caption{$\mathtt{sampleP}(r,\mathcal{B}',\mathcal{X},\mathcal{A}^{S})$}  \label{alg:sampleP2}
\begin{algorithmic}[1] 
\State $\mathcal{R} \gets \emptyset$
\While {$|\mathcal{R}| < r$ } 
\State $s' \sim \mathrm{Unif}(\mathcal{B}')$  \label{line:sampleState}
\If{$s'$ is not covered by $\mathcal{A}^{S}$}   \label{line:rejectState}
\State $\mathcal{T}_{s'} \gets \mathtt{covers}(s')$ \label{line:sampleBox}
\State $i \sim \mathrm{Unif}(\mathcal{T}_{s'})$ // {\scriptsize uniformly random sample from $\mathcal{T}_{s'}$}
\State $\mathcal{R} \gets \mathcal{R} \cup i$
\EndIf
\EndWhile
\State return $\mathcal{R}$
\end{algorithmic}
\end{algorithm}

\begin{algorithm}
\caption{$\mathtt{update}(\Delta_{f},\Phi,i',\mathcal{A}^{S},\mathcal{X})$} \label{alg:update}
\label{alg-dec}
\begin{algorithmic}[1]
\For {$s' \in \Phi(i')$}
\State $\mathcal{T}_{s'} \gets \mathtt{covers}(s')$
\State $\Delta_{f}(i|\mathcal{A}^{S}) = \Delta_{f}(i|\mathcal{A}^{S}) - 1  \ \forall i \in \mathcal{T}_{s'}  $
\State $\Phi(i|\mathcal{A}^S) \gets \Phi(i|\mathcal{A}^S) \setminus s'  \ \forall i \in \mathcal{T}_{s'}$ 
\EndFor
\State return $\langle \Delta_{f}, \Phi \rangle$
\end{algorithmic}
\end{algorithm}

In general, sampling from such a distribution would be difficult, but with PCF we can do this efficiently. Algorithm \ref{alg:sampleP2} first uniformly randomly samples a particle from the belief. If the particle is not covered by $\mathcal{A}^{S}$, then it uses tile coding to find the set of pixel boxes $\mathcal{T}_{s'}$ that cover $s'$ and adds a pixel box uniformly randomly from $\mathcal{T}_{s'}$. This is repeated until $r$ pixel boxes are added to $\mathcal{R}$. 

At the end of each iteration, PartiMax calls $\mathtt{update}$ (Algorithm \ref{alg:update}), which updates $\Delta_{f}(i|\mathcal{A}^{S})$ and $\Phi(i|\mathcal{A}^{S})$ for every $i \in \cup_{s' \in \Phi(i'|\mathcal{A}^{S})} \mathcal{T}_{s'}$. It starts by iterating over the particles $s'$ in $\Phi(i'|\mathcal{A}^{S})$ and for each particle uses the tile coding to find $\mathcal{T}_{s'}$. For every pixel box $i \in \mathcal{T}_{s'}$,  $\Delta_{f}(i|\mathcal{A}^{S})$ is decremented and $s'$ is removed from $\Phi(i|\mathcal{A}^{S})$, to account for the fact that $i'$ now covers $s'$ and thus the marginal gain of $i$ is reduced. The marginal gain of every other $i$ remains unchanged. Similarly, $\Phi(i|\mathcal{A}^{S})$ is updated by subtracting $s'$ from $\Phi(i|\mathcal{A}^{S})$ for every $i$ in $\mathcal{T}_{s'}$. 

\vspace{-2mm}
\section{ANALYSIS}
We now establish bounds on the cumulative error of PartiMax that are independent of the problem size.  We start with a lemma that shows that the probability of adding $i$ to $\mathcal{R}$ via Algorithm \ref{alg:sampleP2} is directly proportional to the marginal gain of $i$.

\begin{lemma} \label{lem:proport}
Let $i = \mathtt{sampleP}(1,\mathcal{B}',\mathcal{X},\mathcal{A}^{S})$ then $\Pr_{\mathcal{A}^{S}}(i = i_j) = c\Delta(i_j|\mathcal{A}^{S})$, where $c = \frac{1}{tm'}$ is a constant where $t$ is the number of tilings and $m$ is the number of particles in $\mathcal{B}'$ that are not covered by $\mathcal{A}^S$.
\end{lemma}
\begin{proof}
The probability that a given pixel box $i_j$ is sampled in $\mathcal{T}_{s'}$ (in line \ref{line:sampleBox}) is the number of particles covered by $i_j$ in $\mathcal{B}'$ that are not covered by $\mathcal{A}^{S}$, which is $\Delta_{f}(i_j|\mathcal{A}^{S})$:
\begin{equation}
\Pr(i_{j} \in \mathcal{T}_{s'}) = \frac{1}{m} \Delta_{f}(i_j | \mathcal{A}^{S}).
\end{equation}
Since there is exactly one pixel box that covers a given state in each of the $t$ tilings, the total number of pixel boxes, that is the size of $\mathcal{T}_{s'}$ is $t$. Since Algorithm \ref{alg:sampleP2} samples uniformly randomly from $\mathcal{T}_{s'}$, then the probability of selecting $i_j$ from $\mathcal{T}_{s'}$ is 
$\frac{1}{|\mathcal{T}_{s'}|} = \frac{1}{t}$. Thus,   
\begin{equation}
\Pr(i = i_{j}) = \frac{1}{t} \frac{1}{m} \Delta_{f}(i_j | \mathcal{A}^{S}). 
\end{equation}
\end{proof}

Next, we show that PartiMax is guaranteed to be near-optimal.  
\begin{theorem} \label{th:genOpt}
Let $F$ be a set function over a collection of sets $\mathcal{A}^+ = \{\mathcal{A}_1, \mathcal{A}_2 \dots \mathcal{A}_v \}$ and let $\mathcal{A}^{*} = \argmax_{\mathcal{A} \in \mathcal{A}^+}F(\mathcal{A})$, let $\mathcal{A}' = \arg\max_{\mathcal{A} \in \mathcal{R}} F(\mathcal{A})$, such that $\mathcal{R}$ is formed by sampling $r$ sets from a probability distribution such that probability of sampling $\mathcal{A}$ is $\Pr(\mathcal{A}) = \frac{1}{c}F(\mathcal{A})$, where $c$ is a scalar constant, such that, $\frac{c}{F(\mathcal{A}^*)}  - 1\leq r $. Then,
\begin{equation}
F(\mathcal{A}^*) - \mathbb{E}F(\mathcal{A}') \leq (\frac{r}{1+r})^r F(\mathcal{A}^*).
\end{equation}
\end{theorem}
\begin{proof}
Let $p_1, p_2 \dots p_v$ denote $P(\mathcal{A}_1), P(\mathcal{A}_2) \dots P(\mathcal{A}_v)$ respectively. Also without loss of generality, we assume $p_1 \geq p_2 \geq \dots p_v$. Consequently, it follows, $F(\mathcal{A}_1) \geq F(\mathcal{A}_2) \geq \dots \geq F(\mathcal{A}_v) $. Note $\mathcal{A}^* = \mathcal{A}_1$.
The expected value of $F(\mathcal{A}')$ is at least as much as: 
\begin{equation}
\mathbb{E}[F(\mathcal{A}')] \geq (1 - (1 - p_1)^r)F(\mathcal{A}_1). 
\end{equation}
The term on the right corresponds to the case, when $\mathcal{A}_1$ is sampled at least once in $\mathcal{R}$, then we are guaranteed to get $\mathcal{A}' = \mathcal{A}_1$. The rest of the cases when $\mathcal{A}_1$ is not sampled in $\mathcal{R}$, we ignore, thus giving us the above bound on $\mathbb{E}[F(\mathcal{A}')]$.
Thus, 
\begin{equation}
F(\mathcal{A}_1) - \mathbb{E}[F(\mathcal{A}')] \leq F(\mathcal{A}_1) - (1 - (1 - p_1)^r)F(\mathcal{A}_1).  \nonumber
\end{equation}
Since $cp_1 = F(\mathcal{A}_1)$, the above equation can be written as: 
\begin{equation}
F(\mathcal{A}_1) - \mathbb{E}[F(\mathcal{A}')] \leq c p_1 - (1 - (1 - p_1)^r)c p_1.
\end{equation}
On differentiating the right hand side with respect to $p_1$ and equating it to zero, we find that the maxima of right hand size occurs at $p_1 = (1/(r+1))$. Also, since $r \geq \frac{c}{F(\mathcal{A}^*)} - 1$, substituting this in the above equation we get, 
\begin{equation}
F(\mathcal{A}_1) - \mathbb{E}[F(\mathcal{A}')] \leq (r/(r+1))^r F(\mathcal{A}_1).  \nonumber
\end{equation}
\end{proof}

The above theorem\footnote{This is a corrected version of the original UAI paper available at: http://auai.org/uai2017/proceedings/papers/130.pdf. The condition on $r \geq \frac{c}{F(\mathcal{A}^*)} - 1$ is missing from the original version.} guarantees that, granted access to a probability distribution such that $\Pr(\mathcal{A}) = cF(\mathcal{A})$, there exists a tight theoretical guarantee for selecting $\mathcal{A}' = \arg\max_{\mathcal{A} \in \mathcal{R}} F(\mathcal{A})$, independent of the problem size.
Directly applying Theorem \ref{th:genOpt} and Lemma \ref{lem:proport} yields the following lemma, which shows that the marginal gain of PartiMax $\Delta_f(i'|\mathcal{A}^S)$ in each iteration is at least $(1 - (\frac{r}{r+1})^r) \Delta_f(i^*|\mathcal{A}^S)$, where $i^* = \argmax_{i \in \mathcal{X}\setminus\mathcal{A}^S}\Delta(i|\mathcal{A}^S)$ and $r \geq \frac{tm}{2} - 1$.

\begin{lemma} \label{oneItGain}
Let $i^* = \arg\max_{i \in \mathcal{X} \setminus \mathcal{A}^{S}} \Delta_{f}(i|\mathcal{A}^{S})$, $r \geq \frac{tm}{2} - 1$ and let $i' = \argmax_{i \in \mathcal{R}} \Delta_{f}(i|\mathcal{A}^{S})$, where $\mathcal{R} = \mathtt{sampleP(r,\mathcal{B}',\mathcal{X},\mathcal{A}^{S})}$, $r \geq \frac{tm}{2} - 1$, then,
\begin{equation}
\Delta_{f}(i^*|\mathcal{A}^{S}) - \mathbb{E}\Delta_{f}(i'|\mathcal{A}^{S}) \leq (\frac{r}{r+1})^{r}\Delta_{f}(i^*|\mathcal{A}^{S}).  \nonumber
\end{equation}
\end{lemma}
\begin{proof}
Using $\mathcal{X} \setminus \mathcal{A}^{S}$ as $\mathcal{A}^+$, $i^*$ as $\mathcal{A}^*$, $i'$ as $\mathcal{A}'$ and applying Theorem \ref{th:genOpt} and Lemma \ref{lem:proport} yields the desired result.
\end{proof}
\vspace{-4mm}
Lemma \ref{oneItGain} in turn yields the following theorem for $r \geq \frac{tm}{2} - 1$: 
\begin{theorem}
\begin{equation}
\mathbb{E} [f(\mathcal{A}^S)] \geq (1 - e^{-1} - (r/(r+1)^r))f(\mathcal{A}^*).
\end{equation}
\end{theorem}
\vspace{-4mm}
\begin{proof} 
Let $\mathcal{A}^{*} = \{i_1^*, i_{2}^{*}, \dots ,i_k^{*} \}$ and $\mathcal{A}^{S}_m = \{i^{S}_1, i^{S}_2, \dots ,i^{S}_m \}$ be the solution returned by PartiMax after $m \leq k$ iterations. Let $\mathcal{A}^{P} = \mathcal{A}^{*} \setminus \mathcal{A}_{m}^{S} =\{i^{P}_1, \dots ,i^{P}_j\} $ and let $\mathcal{A}^{P}_l$ be the first $l$ elements of $\mathcal{A}^{P}$, with $\mathcal{A}^{P}_0 = \emptyset$. Note that $\mathbb{E}f(\mathcal{A}^{*} \cup \mathcal{A}^{S}_{m})$ can be expressed as: 
\begin{equation}
\mathbb{E}[f(\mathcal{A}^{*} \cup \mathcal{A}^{S}_{m})] = \mathbb{E}[f(\mathcal{A}^{S}_m)] + \sum_{l = 1}^{j}\mathbb{E}[\Delta_{f}(i^{P}_l|\mathcal{A}^{S}_m \cup \mathcal{A}^{P}_{l-1} \})].   \nonumber
\end{equation}
$f$ is monotonic, $\mathbb{E}[f(\mathcal{A}^{*} \cup \mathcal{A}^{S}_{m})] \geq \mathbb{E}[f(\mathcal{A}^{*})]$, and by submodularity, 
$\sum_{l=1}^{j}\mathbb{E}[\Delta_{f}(i^{P}_{l}|\mathcal{A}^{S}_m)]  \geq \sum_{l = 1}^{j}\mathbb{E}[\Delta_{f}(i^{P}_l|\mathcal{A}^{S}_m \cup \mathcal{A}^{P}_{l-1} \})]. $ Thus,
\begin{equation} \label{eq:p2}
\mathbb{E}[f(\mathcal{A}^{S}_m)] + \sum_{i \in \mathcal{A}^P}\mathbb{E}[\Delta_{f}(i|\mathcal{A}^{S}_m)] \geq f(\mathcal{A}^{*}).
\end{equation}
From Lemma \ref{oneItGain}, $\mathbb{E}[f(\mathcal{A}^{S}_{m+1}) - f(\mathcal{A}^{S}_{m})] \geq \Delta_{f}(i^*|\mathcal{A}^{S}_{m}) - (\frac{r}{r+1})^{r}\Delta_{f}(i^*|\mathcal{A}^{S}_{m})$. Also, since $|\mathcal{A}^{P}| \leq k$, 

\begin{equation}
\begin{split}
\mathbb{E}[f(\mathcal{A}^{S}_{m+1}) - f(\mathcal{A}^{S}_{m})] + (\frac{r}{r+1})^{r}\Delta_{f}(i^*|\mathcal{A}^{S}_m) )  \\  \geq \frac{1}{k} [f(\mathcal{A}^{*}) - \mathbb{E}[f(\mathcal{A}^{S}_m)] ].
\end{split}
\end{equation}
By induction on $m$ the desired result can be obtained. (\citep{submodSurvey, lazier, satsangi15})
\end{proof}

The above theorem establishes a bound on the error of PartiMax that is independent of the size of the problem and thus remains tight even for large values of $n$. Furthermore, the above result shows that, as the size of $\mathcal{R}$ increases, PartiMax's performance is guaranteed to converge to that of greedy maximization. 

While we have shown these results for PartiMax for selective detection, Theorem \ref{th:genOpt} is applicable to any problem that involves maximization over a set function where we can sample from a probability distribution such that the probability of sampling a subset $\mathcal{A}$ is directly proportional to the value of that subset specified by the set function $F$. Also note that Theorem \ref{th:genOpt} does not make any assumptions about $F$ and is applicable to any set function, submodular or not. 

\section{EXPERIMENTS}
\label{sec:exp}
\vspace{-2mm}

We evaluated PartiMax on a dataset containing approximately 2100 trajectories of people recorded by a camera taking $5120 \times 3840$ resolution images running at 6 frames per second \citep{Schutte2016}. The trajectories were generated using the \emph{ACF} detector \citep{Dollar2014} and in-camera tracking \citep{Schutte2016}. These tracks were used to learn the \emph{motion model} of the people walking in the scene, as described below. 

We model the state $s$ as the person's position and velocity, $s = \langle x,y,v_x,v_y \rangle$, where $x$ and $y$ describe position and $v_x$ and $v_y$ describe velocity.  Both $x$ and $y$ are integers in $\{0,\ldots,5000\}$. We use a motion model that predicts the next position as: 
\begin{equation}
	x_{next} = x_{curr} + v_{x}^{curr} + \mathcal{N}(0,\sigma_{x}^2),
\end{equation}
for $x$ and analogously for $y$. We use a maximum likelihood estimate of $\sigma_x^2$ learned from the data. 

Each pixel box was $180 \times 180$ and each tiling had a $60 \times 30$ offset from the previous one. This offset was chosen because it is the size of the average \emph{bounding box} required to bound a detected person in the scene. This setup yields approximately 7000 pixel boxes per image. 

We assume access to a detector that determines with 90\% accuracy whether a person is located within a given pixel box and gives a noisy observation about the location of the person if detected. Using the motion model and this detector, we maintain a particle belief $\mathcal{B}$ about the person's location using an unweighted particle filter with 250 particles. Multi-person tracking uses a separate particle filter for each person. 

In our experiments, each algorithm selects $k$ pixel boxes to which to apply the detector. To evaluate its performance, we sample a test trajectory from the dataset and try to track the person's movement, starting with a random belief $\mathcal{B}$ and updating it at each timestep using the observations generated from the selected pixel boxes. At each timestep, the agent is asked to predict the position of the person in the scene and gets a reward of +1 for correct predictions and 0 otherwise. Performance is quantified as the total cumulative reward aggregated by the agent at the end of a trajectory over a series of 50 timesteps. 
The experiments were run for over 140 trajectories for 8 independent runs for 1 person tracking and 6 independent runs for 3 and 5 person tracking. 

As a baseline, we compare against an efficient version of greedy maximization (GM+PCF) (in red in plots) that employs tile coding to maintain the particle coverage of each pixel box. GM+PCF is the same as PartiMax but, instead of selecting the pixel box with the highest particle coverage, in each iteration from $\mathcal{R}$, GM+PCF selects it from $\mathcal{X}$. A naive implementation of greedy maximization that computes the particle coverage of each pixel box in every iteration by going over the entire belief was too slow for a complete run and required 160 seconds to select $k=40$ from $n=7200$ for one-person tracking. GM+PCF returns the same solution as greedy maximization but is faster. Simple baselines like downsampling are not useful as the tracking system must still process thousands of pixel boxes even if the image is downsampled by a factor of 4 or 8. Furthermore, downsampling precludes detection of high level features about the person like the color of his/her clothes, etc., thus defeating the purpose of deploying a high resolution camera.

\begin{figure}[t]
\vspace{-2mm}
\includegraphics[scale=0.27]{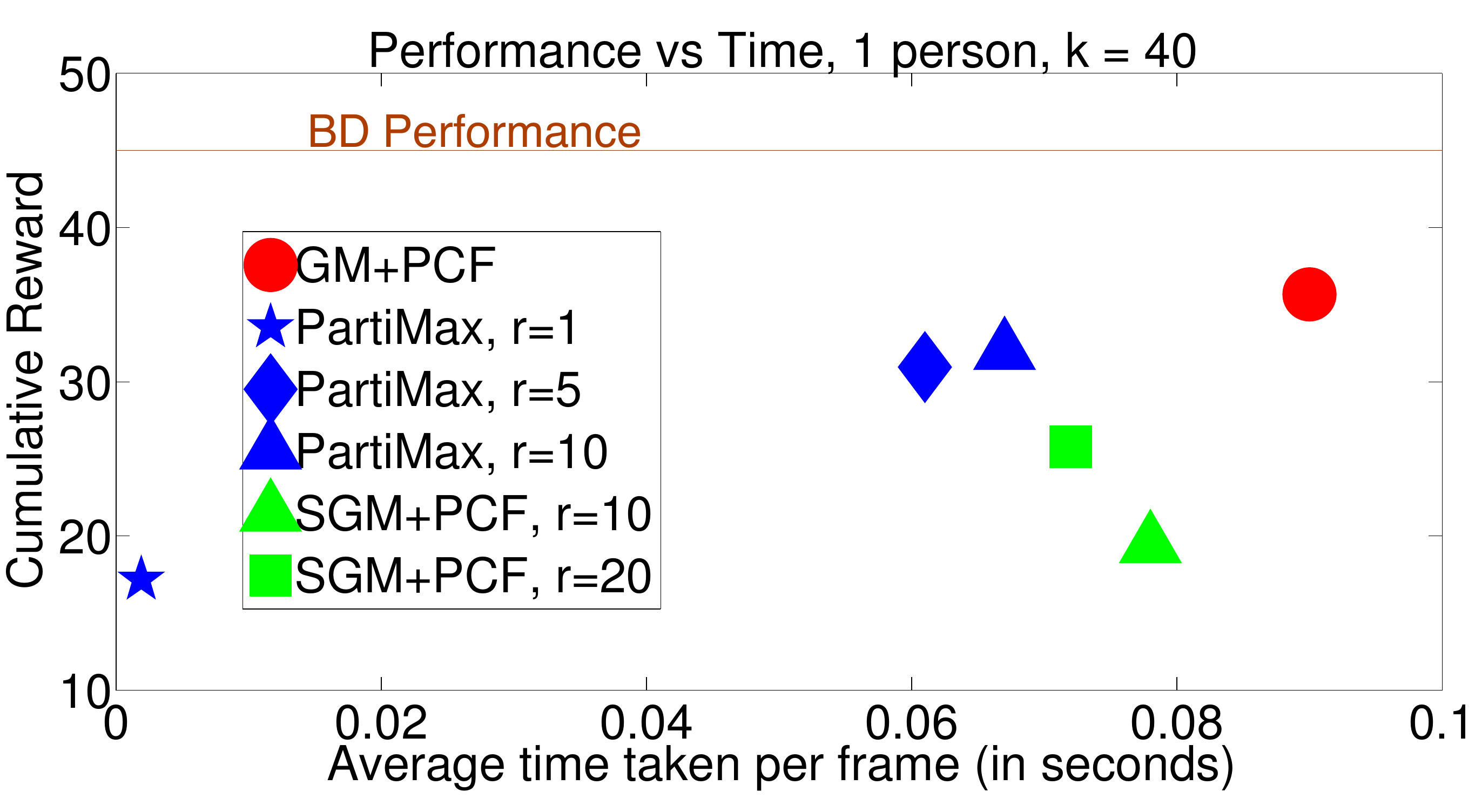} 
\includegraphics[scale=0.27]{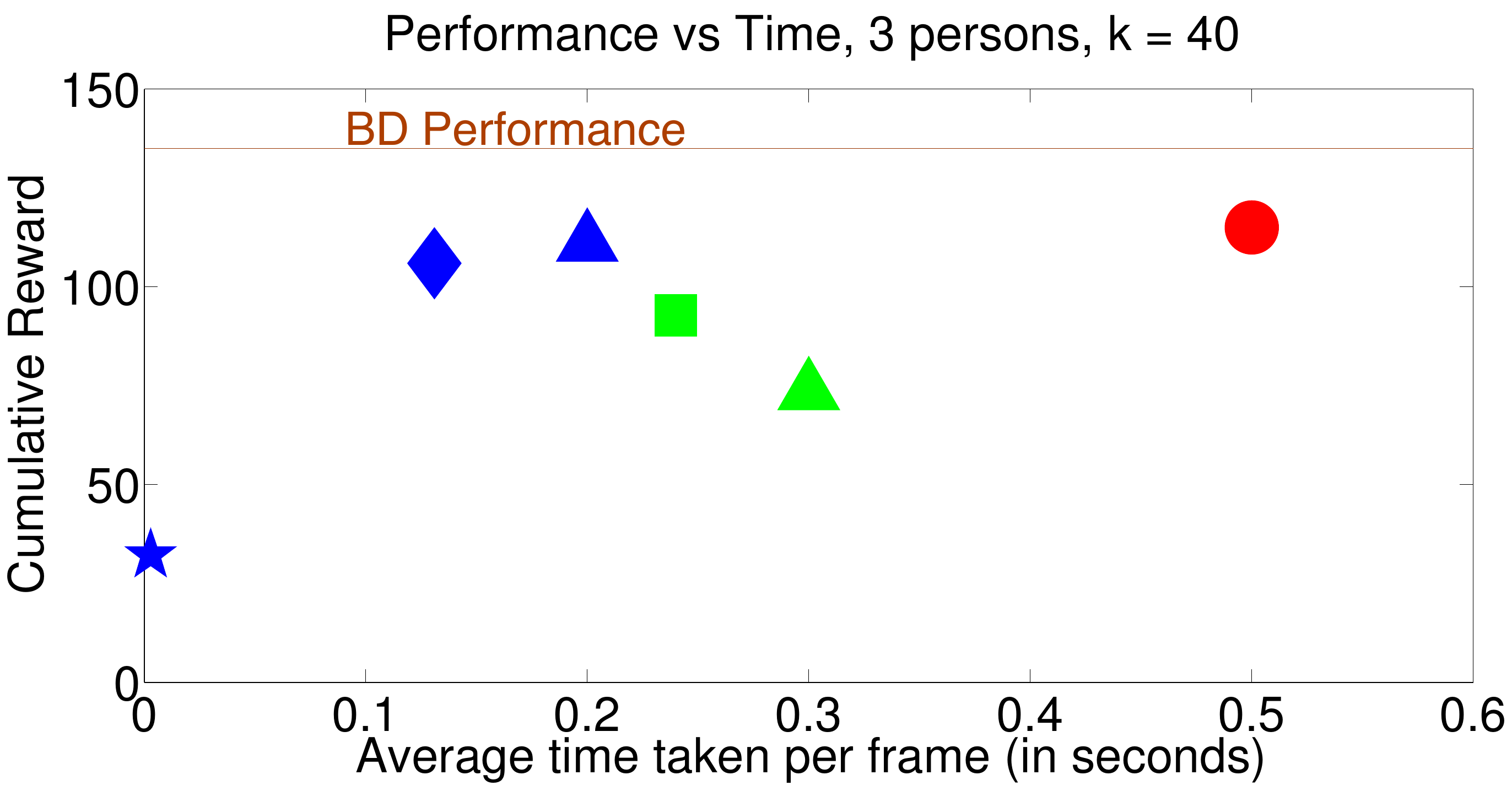} 
\includegraphics[scale=0.27]{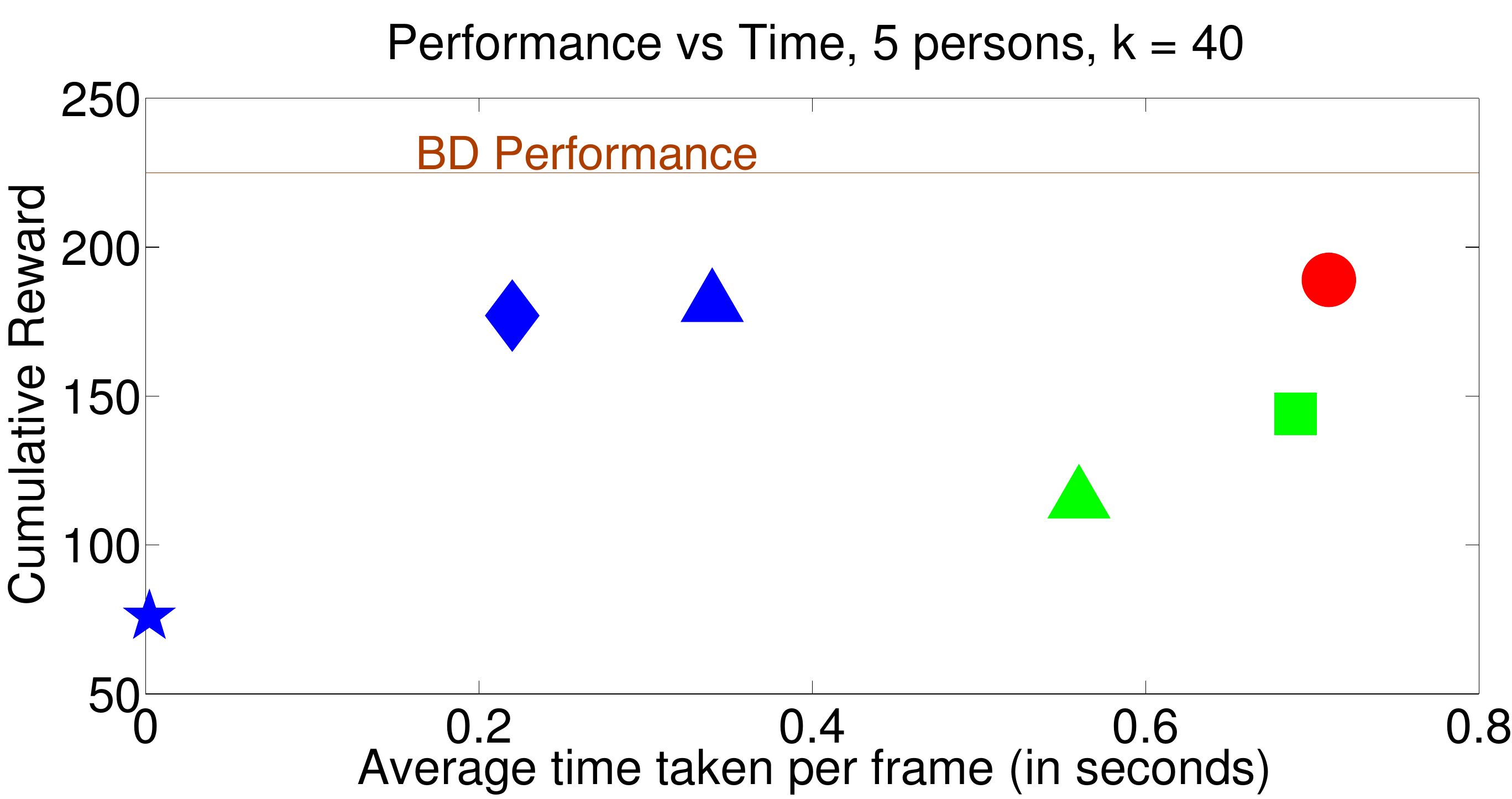}
\vspace{-2mm} 
\caption{Total correct predictions vs.\ CPU time (seconds) for tracking one (top), three (middle) and five  (bottom) people. The closer to the top-left corner, the better.} \label{fig:exps}
\vspace{-4mm}
\end{figure} 

We also compare to stochastic greedy maximization (in green in the plots) that randomly samples a subset $\mathcal{R}$ from $\mathcal{X}$ but employs tile coding to maintain the particle coverage of each pixel box. A naive implementation of stochastic greedy maximization that computes particle coverage of each pixel box from scratch takes around 0.83 seconds for $k=40$ and $r=10$ for one person tracking. The combination of SGM + PCF returns the same solution as stochastic greedy maximization, but faster. 

Figure \ref{fig:exps} shows a detailed comparison between PartiMax, greedy maximization, and stochastic greedy maximization when tracking 1, 3, or 5 people with $k=40$. The $y$-axis shows the cumulative correct predictions averaged over multiple trajectories that the agent made using observations from each algorithm and the $x$-axis shows the time taken by each algorithm to select $40$ out of $7200$ pixel boxes. Thus, the top left corner indicates good tracking performance at a low computational cost. The brown line in the figure shows the tracking performance when the brute force detection is used, that is the person detector is applied to the entire image (except the part containing sky), which takes approximately 2.5 seconds. 

The blue diamond and triangle at the top left corner of each plot show the superior performance and computational efficiency of PartiMax compared to the baselines. PartiMax not only matches the performance of greedy maximization, it does so extremely efficiently with a low value of $r$, thanks to the sampling scheme we propose. Stochastic greedy maximization's tracking performance suffers due to its random sampling, while the computational cost of GM+PCF increases with the number of people. PartiMax combines the best of both of these baselines and performs better both in terms of tracking performance and computational cost. In fact, as the number of people in the scene increases, PartiMax scales much better than any other algorithm. Overall, PartiMax is able to retain 80\% percent of BD's tracking performance but is at least 10 times faster.


\section{CONCLUSIONS \& FUTURE WORK}
This paper proposed a new tracking system that selectively processes only a fraction of an image to track people in real time. We proposed a new algorithm PartiMax that exploits submodularity to quickly identify the most relevant regions in an image. We applied our tracking system to a real-life dataset and showed that it retains 80\% of tracking performance even while processing only a fraction of each image and running in real time. In future we plan to apply PartiMax to sensor selection tasks and other applications that involves maximizing coverage functions.
\vspace{-2mm}
\subsubsection*{Acknowledgements}
 We thank TNO for providing us with the dataset used in our experiments.  We also thank the STW User Committee for its advice regarding active perception for tracking systems. This research is supported by the Dutch Technology Foundation STW (project \#12622), which is part of the Netherlands Organisation for Scientific Research (NWO), and which is partly funded by the Ministry of Economic Affairs.
Frans Oliehoek is funded by NWO Innovational Research Incentives Scheme Veni \#639.021.336.

\bibliographystyle{plainnat}
\bibliography{RealTimeResourceAlloc}

\end{document}